\documentclass{svproc}
\usepackage{amsmath,amssymb,amsfonts}
\usepackage{graphicx}
\usepackage{textcomp}
\usepackage{graphicx,subfigure}
\pdfoutput=1 

\usepackage{algorithm}
\usepackage{algorithmic}

\usepackage{listings}
\usepackage{xcolor}
\usepackage{matlab-prettifier}
\usepackage{tabularx}
\usepackage{graphicx}
\graphicspath{ {./images/} }
\usepackage{accents}
\usepackage[english]{babel}
\usepackage{mathtools}
\mathtoolsset{showonlyrefs}
\usepackage{comment}
\usepackage{enumitem}

\newcommand{\R}{\mathbb{R}}
\newcommand{\cV}{{\mathcal{V}}}

\newtheorem{assumption}{Assumption}

\usepackage{url}

\begin{document}
\mainmatter              
\title{Geometric Foundations of Tuning without Forgetting in Neural ODEs}
\titlerunning{Geometric Foundations of Tuning without Forgetting in Neural ODEs}  
%
\author{Erkan Bayram\inst{1} \and
Mohamed-Ali Belabbas\inst{1} \and Tamer Başar\inst{1}}
\authorrunning{Bayram et al.} 
%
\tocauthor{Erkan Bayram, Mohamed-Ali Belabbas, Tamer Başar}
\institute{Coordinated Science Laboratory, University of Illinois Urbana-Champaign\\ Urbana, IL 61801\\
\email{(ebayram2,belabbas,basar1)@illinois.edu}
\thanks{Research of UIUC authors was supported in part by the ARO Grant W911NF-24-1-0085, NSF-CCF 2106358, ARO W911NF-24-1-0105 and AFOSR FA9550-20-1-0333.}
}
\maketitle              

\begin{abstract}

In our earlier work, we introduced the principle of Tuning without Forgetting (TwF) for sequential training of neural ODEs, where training samples are added iteratively and parameters are updated within the subspace of control functions that preserves the end-point mapping at previously learned samples on the manifold of output labels in the first-order approximation sense. In this letter, we prove that this parameter subspace forms a Banach submanifold of finite codimension under nonsingular controls, and we characterize its tangent space. This reveals that TwF corresponds to a continuation/deformation of the control function along the tangent space of this Banach submanifold, providing a theoretical foundation for its mapping-preserving (not forgetting) during the sequential training exactly, beyond first-order approximation.

\end{abstract}

\section{Introduction}
\label{sec:introduction}
While learning for control has been extensively studied, the inverse problem, using control theory to develop new supervised learning algorithms, remains relatively unexplored. Existing results are largely restricted to analyses based on universal approximation and interpolation theorems, which correspond to controllability over finite ensembles of points in the feature space~\cite{cuchiero2020deep,tabuada2022universal}. Other works, such as~\cite{bonnet2023measure,weinan2018mean}, use mean-field control theory to learn a distribution over the input distribution. Furthermore, the work~\cite{vialard2020shooting} proposes a particle-ensemble parameterization which fully specifies the optimal trajectory of the neural ODE.

In this letter, we consider a neural ODE $\dot x = f(x,u)$ where the learning task is to find a control $u^*$ such that the flow $\varphi_t(u^*,\cdot)$ generated by the neural ODE satisfies
\begin{align}\label{eqn:match}
    R(\varphi_T(u^*,x^i)) = y^i, \quad \forall (x^i,y^i)\in (\mathcal{X},\mathcal{Y}),
\end{align}
for some finite $T\ge 0$, a given readout map $R$, and a finite set of training pairs $(\mathcal{X},\mathcal{Y})$ of size $q$. The map $R(\varphi_T(u^*,\cdot))$ is called the {\em end-point mapping}.

The existing methods for training neural ODEs (i.e. finding $u^*$ that satisfies~\eqref{eqn:match}) have limitations, such as the $q$-folded method, which scales quadratically with the dataset size and requires retraining from scratch when new points are added~\cite{agrachev2020control}. To address these limitations, in our earlier work, we introduced Tuning without Forgetting (TwF), an iterative training algorithm for neural ODEs ~\cite{bayram2024control}. TwF sequentially incorporates new points into the training set (i.e. adding one point at a time) and learns the new pair without forgetting the already learned pairs. In this way, TwF solves the central problem in continual learning: learning new points without forgetting previously acquired knowledge. This property enables neural ODEs to be applied across diverse learning problems and settings.


TwF iteratively updates the control function $u$ (i.e. the parameters of neural ODE) to a new control $\tilde u$ to steer the latest introduced training point $x^{j+1}$ to its target $y^{j+1}$, while keeping the end-point mapping at all previously learned pairs $\{(x^i,y^i)\}_{i=1}^{j}$ invariant at every iteration. More precisely, once the $(j+1)$th point is added to the training set, TwF restricts each update to the set of controls that satisfy the invariance condition
\begin{align}\label{eqn:intro}\\[-2em]
    R(\varphi_T(u, x^i)) = y^i, \quad \forall i < j+1,\\[-1.5em]
\end{align} until it reaches a control $\tilde u$. However, the existence of a sequence of control functions (initiated from $u$ and converging to $\tilde u$) that exactly satisfies~\eqref{eqn:intro} at each iteration remains an open question. Therefore, TwF enforces~\eqref{eqn:intro} only in a first-order approximation sense, which may lead to error accumulation over iterations.

Despite this theoretical limitation, TwF has already demonstrated practical impact. For example, it enables the training of robust neural ODEs~\cite{bayram2025control} that are resilient to control disturbances by solving a nonconvex–nonconcave minimax problem over an infinite-dimensional function space. It has also been applied in federated learning~\cite{rai2025control} to mitigate heterogeneity in distributed data.

From a theoretical perspective, classical control theory offers a natural framework to address this issue. In particular, the continuous deformation of the control function under fixed boundary conditions has been extensively studied. Sontag~\cite{sontag2002control} introduced the generation of nonsingular loops (i.e. the linearized controllability along suitable trajectories), while Sussmann~\cite{sussmann1993continuation,sussmann1992new} proposed a point-to-point path-finding method based on continuation/deformation as an alternative to shooting method, establishing the existence of a family of control functions realizing the same point-to-point mapping, which can be parameterized by a single element.

Inspired by these results, we show that the set of controls that learns a training pair $(x^i,y^i)$ 
forms a Banach submanifold if the model has the linearized controllability property (i.e. the first-order controllability of a nonlinear system). Then, we prove that the subspace of controls satisfying~\eqref{eqn:intro} forms a Banach submanifold of finite codimension under the strong memorization property, and we characterize its tangent space. Importantly, we show that the control function can be continuously deformed along this tangent space, thereby preserving the end-point mapping at the previously learned points.



{Our main contributions are as follows:}
\begin{itemize}
    \item We formalize Tuning without Forgetting (TwF) as a framework to sequentially train neural ODEs while preserving the end-point mapping at previously learned training points.
    \item We prove that the subspace of the control function that preserves the end-point mapping at given set of initial points forms a Banach submanifold of the space of bounded functions of finite codimension and we explicitly characterize its tangent space.
    \item We show that TwF can be interpreted as a continuous deformation of the control function along this tangent space, providing a theoretical foundation for its mapping-preserving (not forgetting) during the sequential training exactly, beyond first-order approximation.
\end{itemize}

The rest of the paper is organized as follows. In Section~\ref{sec:prelim}, we provide the preliminaries and formalize the learning problem as a multi–motion planning task. Section~\ref{sec:main_results} presents the main results, beginning with a geometric characterization of the control sets under the linearized controllability properties. We then establish the strong memorization property. Building on these results, we revisit the Tuning without Forgetting (TwF) algorithm from a geometric viewpoint. Section~\ref{sec:proof} is devoted to the proof of the main theorem. Section~\ref{sec:conl} summarizes the contributions and discusses directions for future work.

\section{Preliminaries}\label{sec:prelim}


 Consider the paired set $(\mathcal{X},\mathcal{Y}) = \{(x^i, y^i)\}_{i=1}^q$, where each $x^i \in \mathbb{R}^{n}$ is an initial point, and its corresponding $y^i \in \mathbb{R}^{n_o}$ is a target. The elements of the input ensemble $\mathcal{X}$ are assumed to be pairwise distinct, i.e., $x^i \neq x^j$ for $i \neq j$. Let $\mathcal{I}:=\{1,2,\cdots,q\}$ be an index set that labels the entries of $\mathcal{X}$. Let $\mathcal{X}^j=\{ x^i \in \mathcal{X} | i = 1,2,\cdots,j\} \subseteq \mathcal X$,  called {\em sub-ensemble} of ${\mathcal X}$. Let $\mathcal{Y}^j$ be the corresponding batch of labels for $j>0$. Let $\mathcal X^0$ and $\mathcal{Y}^0$ be the empty sets.
 
 Let $\cV$ be the space of bounded functions from $[0,1]$ to $\mathbb{R}^p$, precisely, $\cV := L^{\infty}([0,1], \mathbb{R}^{p})$. We take the system:
\begin{equation}\label{eqn:control_system}
    \Dot{x}(t) = f( x(t) , u(t) )
\end{equation}
where $x(t)\in\mathbb{R}^{\bar{n}}$ is the state vector at time $t$ and $f(\cdot)$ is a smooth vector field on $\mathbb{R}^{\Bar{n}}$ and {$u(t) \in \cV$. The flow of this system defines the map  
$$\varphi_t(u,x): \cV \times \mathbb{R}^{\bar{n}}  \to \mathbb{R}^{\bar{n}}$$ 
which assigns an initial state $x$ and a control $u$ to the solution of \eqref{eqn:control_system} at $t$, that is, it yields the trajectory $t \mapsto \varphi_t(u,x^i)$ of~\eqref{eqn:control_system} with control $u$ and initialized at $x^i$ at $t=0$. We suppress the subscript $1$ in the notation at $t=1$ for simplicity. Suppose~\eqref{eqn:control_system} has uniformly bounded $\frac{\partial f(x,u)}{\partial u}$ and $\frac{\partial f(x,u)}{\partial x}$ for $t \in [0,1]$ where $x=\varphi_t(u,x^i)$. We introduce the following map to embed an $n$-dimensional space into an $\bar n$-dimensional one where $\bar n \geq n$:
\begin{align}\label{eqn:uplift}
    E: \mathbb{R}^{n} \to \R^{\bar n}: x \mapsto E(x):=(x,0,\ldots,0).
\end{align}
Let $R:\mathbb{R}^{\bar{n}} \to \mathbb{R}^{n_o}$ be a given function, called the {\em readout map} such that the Jacobian of $R$ is of full row rank and $R(\cdot)$ is bounded, linear and $1$-Lipschitz. Note that any projection function satisfies these conditions on $R(\cdot)$ and both functions $E(\cdot)$ and $R(\cdot)$ are independent of the control $u$. We call $R(\varphi(u,E(\cdot)))$ {\em the end-point mapping}, that is, we have the following end-point mapping for a given $u$ at $x^i$:
\begin{align}\\[-2em]
    x^i \in \mathbb{R}^{n} \xrightarrow{E} \Bar{x}^i \in \mathbb{R}^{\Bar{n}}  \xrightarrow{\varphi_T(u,\cdot)} \Bar{y}^i \in \mathbb{R}^{\bar{n}} \xrightarrow{R} y^i  \in \mathbb{R}^{n_o} 
\end{align}

{The learning problem turns into finding a control function $u$ that performs motion planning for initial points $\bar{x}^i$ to a point in the set $\{ \bar{y}^i \in \mathbb{R}^{\bar{n}} : R(\bar{y}^i) = y^i\}$ for all $i \in \mathcal{I}$ \textit{simultaneously}. } We define the set $S_i:=\{ \bar{y} \in \mathbb{R}^{\bar{n}} : R(\bar{y}) = y^i\}$ for $i \in \mathcal{I}$, that is, $S_i=R^{-1}(y^i)$. Now, the learning is the problem of finding $u$ so that 
\begin{equation}
       \varphi(u,E(x^i))=\bar{y}^i \mbox{ where  } \bar{y}^i \in S_i, \forall i \in \mathcal{I},  
\end{equation}
for a fixed uplift function $E$ and readout map $R$. We call this problem a {\em multi-motion planning problem}. In Figure~\ref{fig:mmp}, we provide an example of a classification problem, formulated as a multi-motion planning problem.

\begin{figure}[h!]
\centering
\includegraphics[width=0.8\textwidth]{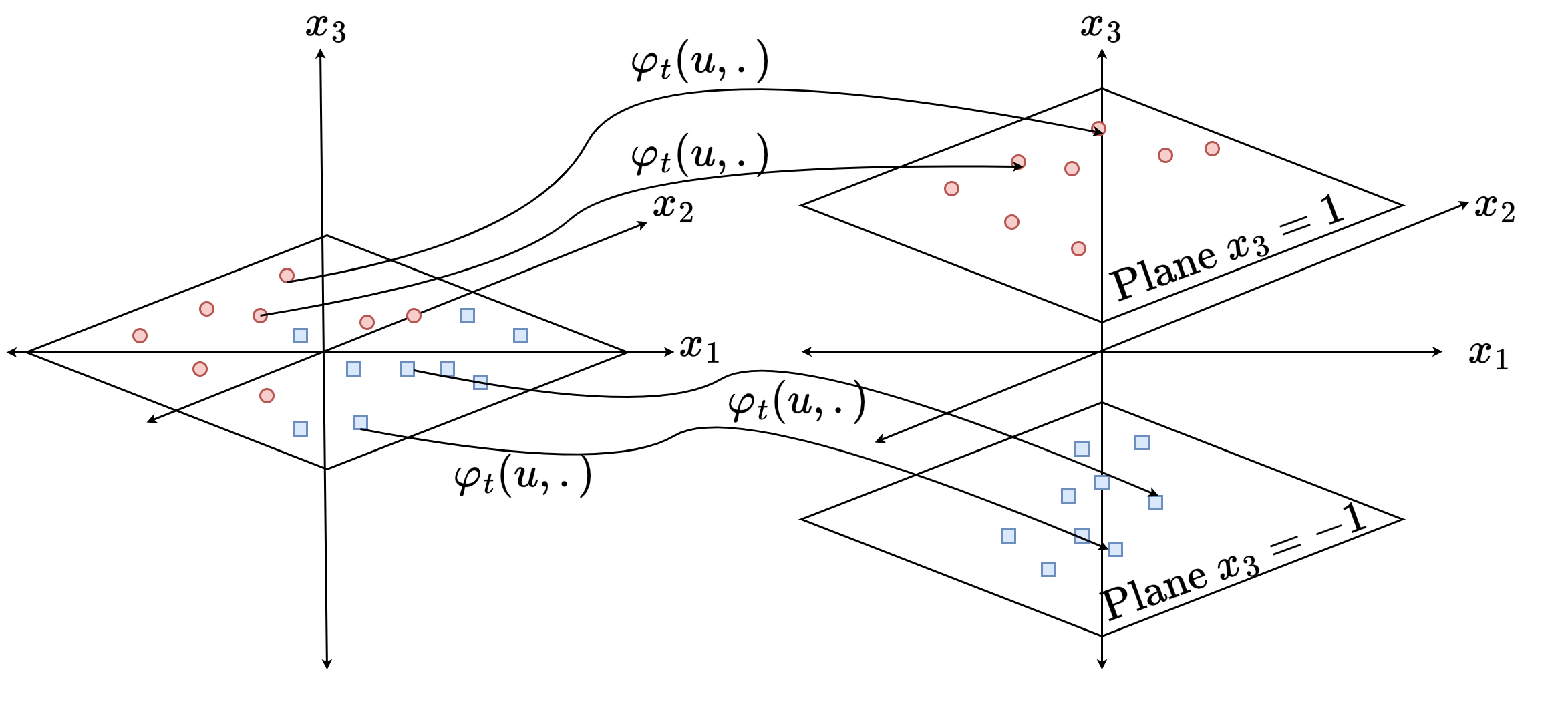}
\caption{Illustration of the learning problem as a multi-motion planning task. Initial points in the input space (left) are mapped by the flow $\varphi_t(u,\cdot)$ induced by a control function $u$ to their corresponding target sets $S_i=R^{-1}(y^i)$ in the output space (right). In this example, the red points are mapped to the plane $x_3=1$, and the blue points are mapped to the plane $x_3=-1$. Each point must reach its assigned plane while respecting the output constraints $R^{-1}(y^i)$ simultaneously. so that the points are classified according to their target planes. For any new test point, the same control function $u$ is applied, enabling the trained model (i.e., the control system with control $u$) to perform the classification task.}
    \label{fig:mmp}
\end{figure}

To formalize this problem, we next define the memorization property for a dynamical system.

\begin{definition}[Memorization Property]\label{defn:mp} Assume that a paired set $(\mathcal{X},\mathcal{Y})$, a fixed readout map $R$ and an up-lift function $E$ are given. The control $u$ is said to have {\em memorized the ensemble $(\mathcal{X},\mathcal{Y})$} for the model $\dot{x}(t)=$ $f(x(t), u(t))$ if the following holds for a finite $T\geq0$:
\begin{align}\label{eqn:defn_fixed_ensemble}\\[-2em]
    R( \varphi_T(u,{E}(x^i)))=y^i, \forall x^i \in \mathcal{X},
\end{align}
\end{definition} 

In other words, the dynamical system has memorized the ensemble if the end-point mapping maps each $x^i \in \mathbb{R}^{n}$ to the corresponding $y^i \in \mathbb{R}^{n_o}$.

{ We introduce the following subspace of our control space $\cV$ for $i\in\mathcal{I}$:
\begin{equation}
    U(x^i,y^i):=\{ u \in \cV | \varphi(u,E(x^i)) \in R^{-1}(y^i) \}
\end{equation}
Then, we need to show that there exists a control function $u \in \bigcap_{i=i}^q U(x^i,y^i)$ to prove that the model has memorization property for the ensemble $(\mathcal{X},\mathcal{Y})$. 

We define the set $$\Delta^q :=\{ [ E(x^1)^\top , \cdots , E(x^q)^\top ]^\top \in E(\mathbb{R}^{\bar{n}})^{q}  \rvert E(x^i) = E(x^j) \mbox{ for } i \neq j \} .$$
Let
$ (\mathbb{R}^{\bar{n}})^{(q)} := (\mathbb{R}^{{\bar{n}}})^q \setminus \Delta^q$ be the complement of $\Delta^q$ on $(\mathbb{R}^{\bar{n}})^q$. Then, we define the set of control vector fields of the $q$-folded system of the model~\eqref{eqn:control_system},
\begin{align}
    \mathcal{F} = \{[f^\top(x,u), \cdots, f^\top(x,u)]^\top  \in \mathbb{R}^{\bar{n} q} | u \in \cV
      \}
\end{align}
In words, we copy $\bar{n}$ dimensional dynamics in~\eqref{eqn:control_system} $q$-times, creating an $\bar{n}q$-dimensional vectors $F(x(t),u(t)) \in \mathcal{F}$. 

We have a sufficient condition for the existence of a control function $u \in \bigcap_{i=1}^q U(x^i, y^i)$ as an application of the Chow-Rashevsky theorem~\cite{brockett2014early}:
\begin{lemma}\label{lem:control}
If the set of control vector fields of $q$-folded system of the model~\eqref{eqn:control_system} is bracket-generating in $(\mathbb{R}^{\bar n})^{(q)}= (\mathbb{R}^{\bar n})^q \setminus \Delta^q$, then there exists a control function $u$ such that $R(\varphi(u,E(x^i)))=y^i,  \forall i \in \mathcal{I}$.
\end{lemma}
\begin{proof}
    See~\cite[Proposition 6.1]{agrachev2022control} for the proof of Lemma~\ref{lem:control}.
\end{proof}
{
\begin{remark}
We note that the control function $u$ still belongs to $\cV (=L_{\infty}([0,1], \mathbb{R}^{p}))$, and not $L_{\infty}([0,1], \mathbb{R}^{qp})$. Therefore, it is a stronger notion than the controllability of the model~\eqref{eqn:control_system} at each $x^i \in \mathcal{X}$. \end{remark}
}
For convenience, we assume that $E$ is the identity function, meaning $\overline{n}=n$, but our result holds for any {continous} injective up-lift function given in~\eqref{eqn:uplift}. Thus, we interchangeably use $x^i$ and $E(x^i)$ for any $x^i \in \mathcal{X}$. 

\paragraph{Cost Functional:} We define per-sample cost functional $\mathcal{J}_i(u)$ for a given point $x^i$ as follows:
 \begin{align}\label{eqn:per_sample}
     \mathcal{J}_i(u) = \frac{1}{2}\|R(\varphi(u,{x}^i)) - y^i \|^2 
 \end{align}\\[-1.5em]
 We are interested in the minimization of the functional $\cal J:\cal V \to \mathbb{R}$, including regularization, defined as
\begin{align}\label{eqn:cost_cont}
    \mathcal{J}( u , \mathcal{X} ) := \sum_{i=1}^{q} \| R(\varphi(u,x^i)) - y^i \|^2   + \lambda \int_{0}^{T} |u(\tau)|^2 d\tau
\end{align}
 where $\lambda$ is some regularization coefficient.

\section{Main Results}\label{sec:main_results}

\subsection{Geometry of The Set of Control}

In this work, we revisit the TwF that is an iterative algorithm to find $u \in \bigcap_{i=1}^q U(x^i,y^i)$. The core idea is to restrict updates on the control function to the intersection of feasible sets, i.e. $\bigcap_{i=1}^j U(x^i,y^i)$, and progressively enlarge this intersection as new data arrive $(x^{j+1},y^{j+1})$, allowing the model to learn new points without forgetting prior ones. Suppose that at iteration $k$ the control $u^k(=u_j)$ satisfying:
\begin{align}\\[-3em]
    u^k \in \bigcap_{i=1}^{j} U(x^i,y^i),\\[-2em]
\end{align}
so that all points in $(\mathcal{X}^j,\mathcal{Y}^j)$ are correctly mapped via the end-pint mapping. When a new pair $(x^{j+1},y^{j+1})$ is introduced, the iterative method seeks a control $u^{k+\ell_1}(=u_{j+1})$ such that
\vspace{-0.3em}
\begin{align}\\[-2em]
    u^{k+\ell_1} \in \bigcap_{i=1}^{j+1} U(x^i,y^i),\\[-2.5em]
\end{align}
while the intermediate iterations $u^{k+\ell}$, for $0\leq\ell\leq\ell_1$, remain in
\vspace{-0.3em}
\begin{align}\\[-2em]\label{eqn:intermediate}
    u^{k+\ell} \in \bigcap_{i=1}^{j} U(x^i,y^i),\\[-2em]
\end{align}
ensuring that the end-point mapping at previously learned points is preserved during each update. We provide an overview of the algorithm in Figure~\ref{fig:algo}.\begin{figure}[htp!]
     \centering
     \includegraphics[width=0.63\linewidth]{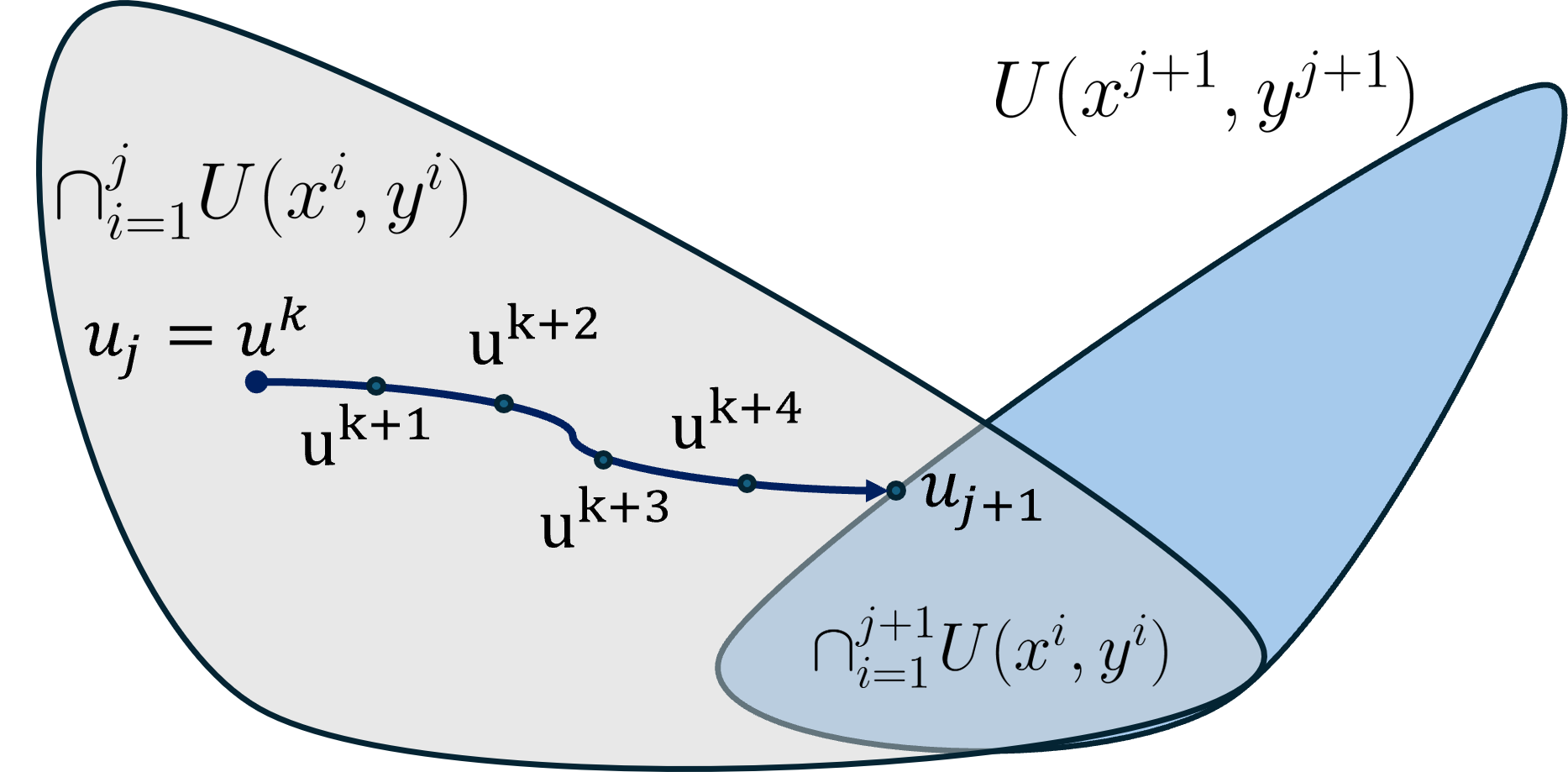}
     \caption{Illustration of the iterative control update process. The gray region represents the set of controls that memorize the sub-ensemble $( \mathcal{X}^j,\mathcal{Y}^j)$, (i.e. $\bigcap_{i=1}^j U(x^i, y^i)$), while the blue region corresponds to the set of controls that memorize the new pair $(x^{j+1},y^{j+1})$ (i.e. $U(x^{j+1}, y^{j+1})$). The intersection of these sets, $\bigcap_{i=1}^{j+1} U(x^i, y^i)$, is the set where $u_{j+1}(=u^{k+\ell_1})$ should lie at the end of iteration for TwF to memorize $(\mathcal{X}^{j+1},\mathcal{Y}^{j+1})$. The dark blue trajectory $u_j,u^{k+1}, u^{k+2}, \dots, u_{j+1}$ shows successive updates obtained by projecting the gradient of the per-sample cost for the new point onto the tangent space of $\bigcap_{i=1}^j U(x^i, y^i)$, ensuring that the end-point mapping at all previous points are preserved.}
     \label{fig:algo}
 \end{figure}

Having a non-empty intersection of the control set $\cap_{i=1}^q U(x^i,y^i)$ is not sufficient to enable the use of gradient methods and does not guarantee the existence of an update sequence satisfying~\eqref{eqn:intermediate}; differentiability of the evolving intersection $\cap_{i=1}^j U(x^i,y^i)$ is also required.

Therefore, we first prove that
$
\bigcap_{i=1}^j U(x^i, y^i)
$
forms a Banach submanifold of $\mathcal{V}$ with finite codimension. This result allows us to characterize its tangent space and derive a closed-form expression for the projection operation onto it. At each iteration, the update in control
$
\delta u^k := u^{k+1} - u^k
$
is chosen as the projection of the first-order variation of the per-sample cost, $\mathcal{J}_{j+1}$, on the tangent space of the set $\bigcap_{i=1}^j U(x^i, y^i)$, ensuring that all iterations $u^k$ remain in the set $\bigcap_{i=1}^j U(x^i, y^i)$.

To address the differentiable properties of the set $\cap_{i=1}^q U(x^i,y^i)$, we need the following properties.
\subsubsection{Linearized Controllability Property}
 Consider the trajectory  $\varphi_t(u,x^i)$. The first-order variation in $\varphi_t(u, x^i)$, denoted by $\delta \varphi_t(u, 
 x^i)=\varphi_t(u+ \delta u, x^i)-\varphi_t(u,x^i)$, obeys a linear time-varying equation, which is simply the linearization of the control system~\eqref{eqn:control_system} around the trajectory $\varphi_t(u, x^i)$. We have the following LTV system:
 \begin{align}\label{eqn:defn_ltv}
    \dot{z}(t) =  \frac{\partial f(x,u)}{\partial x} z(t) + \frac{\partial f(x,u)}{\partial u}  v(t),
\end{align}
where $z(t)=\delta \varphi_t(u,x^i)$, $x(t)=\varphi_{t}(u,x^i)$, $v(t) \in \cV$. Denote the state transition matrix of the system in~\eqref{eqn:defn_ltv} by $\Phi_{(u, x^i)}(1,t)$ for initial point $x^i$ at control $u$. For completeness, we have the following lemma:
\begin{lemma}\label{lem:variation_t}
Suppose that a given control function $u$ has memorized the pair of points $(x^i,y^i)$ for the model~\eqref{eqn:control_system}. Then,
\begin{equation*}
     \delta \varphi_t(u,x^i) = \int_{0}^{t}  \Phi_{(u,x^i)}(t,\tau)  \frac{\partial f( x, u ) }{\partial u} \rvert_{(x=\varphi_{\tau}(u,x^i),u)}  \delta u(\tau)  d\tau.
\end{equation*}
\end{lemma}
See~\cite[Lemma 2]{bayram2024control} for the proof of Lemma~\ref{lem:variation_t}. Then, we can state the following property for the system~\eqref{eqn:control_system}:

\begin{definition}[Linearized Controllability Property]
If, for all $u \in U(x^i,y^i)$, the linear time-varying system in~\eqref{eqn:defn_ltv}  is controllable, then we say that the system~\eqref{eqn:control_system} has the Linearized Controllability Property (LCP) at $x^i$.
\end{definition}    

LCP is equivalent to the first-order controllability of the nonlinear system~\eqref{eqn:control_system} along the trajectory corresponding to the initial state $x^i$ and the control $u$~\cite{sontag2002control}. LCP fails precisely along certain exceptional trajectories, which are referred to in the literature by various names, including ‘abnormal extremals’ or ‘singular trajectories’~\cite{sussmann1993continuation}. In~\cite{strichartz1986sub}, the notion of Strong Bracket Generating (SBG) system is introduced.
It is proven therein that for an SBG system, the only abnormal extremals are the trajectories corresponding to the trivial control $u(t)=0$ for $t \in [0,1]$,~\cite[Theorem 3]{sussmann1993continuation}.

On the one hand, the general conditions ensuring the nonexistence of abnormal extremals for nonlinear systems remain a relatively open problem. On the other hand, in neural ODEs, we often consider over-parameterized architectures; for example, nonlinear dynamics of the form
\[
\dot{x}(t) = \tanh\big(U(t) x(t) + b(t)\big), 
\quad U \in L_\infty\big([0,1], \mathbb{R}^{\bar{n}\times \bar{n}}\big), 
\quad b \in L_\infty\big([0,1], \mathbb{R}^{\bar{n}}\big),
\]
where the control $u$ is the tuple $(U,b)$. In such cases, one can easily verify that the LCP condition is satisfied, thereby ruling out the presence of abnormal extremals along nonsingular trajectories.

\subsubsection{Strong Memorization Property}

 First, recall that two Banach submanifolds $M_1$,$M_2$ of a Banach space $N$ intersect transversally at $x\in N$ if either $x \notin M_1 \cap M_2$ or if $T_x M_1 + T_x M_2 = T_x N$. In words, two manifolds intersect transversally at $x$ either if $x$ does not belong to the intersection, or if the tangent spaces of $M_1$ and $M_2$ at $x$ together span the tangent space of the ambient space $N$. It is known that tranversal intersections are {\em generic}. Now, we can define the following property:

\begin{definition}[Strong Memorization Property]\label{defn:smp}
We say that the model~\eqref{eqn:control_system}  has the strong memorization property for an ensemble $(\mathcal{X},\mathcal{Y})$ if the model has memorization property for $(\mathcal{X},\mathcal{Y})$ and the sets $U(x^i,y^i)$ intersect transversally. It has the strong memorization property of degree $q$ over $\mathbb{R}^{\bar{n}}$ if it has strong memorization property for all set $\{(x^i,y^i)\}_{i=1,\cdots,q}$ with $x^i \in \mathbb{R}^{\bar{n}}$ and $y^i \in \mathbb{R}^{n_o}$. 
\end{definition}

To paraphrase, the strong memorization property states that the control sets \( U(x^i, y^i) \) for \( i = 1, \ldots, q \) have a non-empty intersection and intersect transversally.

\subsubsection{Main Theorem}

Now, we have the following assumptions for our algorithm. Note that systems on manifolds can also be considered. We can replace $\mathbb{R}^{\bar{n}}$ with $\mathcal{M}$ where $\mathcal{M}$ is the maximal subset of $\mathbb{R}^{\bar{n}}$ such that the $q$-folded system is controllable over $\mathcal{M}^q= \mathcal{M} \times \mathcal{M} \times \ldots \mathcal{M}$ ($q$-times). However, this generalization only complicates the notation. Therefore, we restrict our presentation to $\mathbb{R}^{\bar{n}}$ for clarity.


\begin{assumption}[Memorization Property]\label{ass:mem}
 The set of control vector fields of the $q$-folded system of the model~\eqref{eqn:control_system} is bracket-generating in $(\mathbb{R}^{\bar{n} })^{(q)}$. 
\end{assumption}
\begin{assumption}[Linearized Controllability]\label{ass:lpc}
    The system $\dot x(t) = f(x(t), u(t))$ on a manifold $\mathbb{R}^{\bar{n} }$ has the linearized controllability property at all $x^i \in \mathcal{X}$.
\end{assumption}
\begin{assumption}[Strong Memorization]\label{ass:controllable}
The system $\dot x(t) = f(x(t), u(t))$  has the strong memorization property of degree $q$ over $\mathbb{R}^{\bar{n} }$. 
\end{assumption}

\begin{theorem}
\label{thm:control_w_smp} 
Suppose that assumptions (A\ref{ass:mem}), (A\ref{ass:lpc}), and (A\ref{ass:controllable}) hold, and let $(\mathcal{X},\mathcal{Y})$ be a paired set of cardinality $q$. Then, the space of controls $\cap_{i=1}^j U(x^i,y^i)$ is a Banach submanifold of $\cV$ of finite-codimension for $1\leq j \leq q$.
\end{theorem}}
See Section~\ref{sec:proof} for a proof of Theorem~\ref{thm:control_w_smp}. This result establishes that there exist infinitely many controls satisfying the memorization property for \( (\mathcal{X}, \mathcal{Y}) \), and that the collection of such controls forms a smooth Banach submanifold. This geometric structure allows us to introduce our algorithm.

\subsection{Tuning without Forgetting}

Now, we recall the Tuning without Forgetting (TwF) algorithm introduced in~\cite{bayram2024control}. We formalize it by the following definition:

\begin{definition}[Tuning without Forgetting]\label{defn:lwf} Consider an ensemble $(\mathcal{X},\mathcal{Y})$. Assume that the control $u^k$  has memorized the sub-ensemble $(\mathcal{X}^j,\mathcal{Y}^j)$ for~\eqref{eqn:control_system} for some $j<q$. If the update $\delta u^k$ satisfies the following: 
\begin{enumerate}
\item\label{itm:cond_1} $ \mathcal{J}_{j+1}(u^{k}+\delta u^k) \leq \mathcal{J}_{j+1}(u^k) $
{ \item\label{itm:cond_2} $R\left( \varphi(u^{k}+\delta u^k,x^i) \right) =y^i, \forall x^i \in \mathcal{X}^j$ }
\end{enumerate}
then the control function $u^{k+1}(:=u^k+\delta u^k)$  has been {\em tuned} for $\mathcal{X}^{j+1}$ {\em  without forgetting} $\mathcal{X}^j$.
\end{definition}

In~\cite{bayram2024control}, TwF is defined in the sense of a first-order approximation. 
More precisely, the second property of TwF requires selecting $\delta u^k$ such that 
the points in $\mathcal{X}^j$ are mapped to points whose projections onto the output subspace are within 
$o(\delta u^k)$ of their corresponding labels. In Definition~\ref{defn:lwf}, we strengthen this notion by removing the first-order approximation: thanks to Theorem~\ref{thm:control_w_smp}, the property now holds {\em exactly}.

The first property of TwF requires that the per-sample cost for any newly introduced point is nonincreasing under a control update $\delta u^k$. Therefore, we define the first-order variation of the per-sample cost functional $\mathcal{J}_{i}(u)$ at a control $u$ in $\delta u$ as $D_{\delta u} {\mathcal{J}}_{i}(u):= {\mathcal{J}}_{i}(u+\delta u)-\mathcal{J}_i(u)$. Then, we have the following: 
\begin{equation}
    D_{\delta u}{\mathcal{J}}_i(u) := R^\top (\delta  \varphi_t(u,x^{i} ))\left( R( \varphi(u,x^{i})) - y^{i} \right)
\end{equation}
where we recall that $\delta \varphi_t(u, 
 x^i)=\varphi_t(u+ \delta u, x^i)-\varphi_t(u,x^i)$.

The second property of TwF requires that the selection of $\delta u^k$ such that the end-point mapping at already learned samples is fixed. 
 To characterize such $\delta u^k$, we consider the first-order variation of end-point mapping $R(\varphi(\cdot,x^i))$ at control $u^k$ as follows: 
\begin{align}\label{eqn:variation_on_perturb}
    \!\delta R(\varphi(u^k,x^i))\!:=\!R(\varphi(u^k+\delta u^k,x^i))-R(\varphi(u^k,x^i))
\end{align}

To characterize the directions \( \delta u^k \) that keep this variation equal to zero, {we define an operator from the space of bounded functions over the time interval $[0,1]$, $v \in \cV$, to $\mathbb{R}^{n_o}$, mapping a control variation to the resulting variation in the end-point of the trajectory. Based on Lemma~\ref{lem:variation_t}, we have the following:
\begin{align}\label{eqn:operator_L}
    \mathcal{L}_{(u,x^i)}(v) = \left( \int_{0}^{t}  \Phi_{(u,x^i)}(t,\tau)  \frac{\partial f( x, u ) }{\partial u}  v(\tau)  d\tau \right)
\end{align}
where $x=\varphi_{\tau}(u,x^i)$. 
} Then, we define 
$${\ker}(u,x^i):=\operatorname{span}\{ \delta u \in \cV \mid  R( \mathcal{L}_{(u,x^i)}(\delta u))= 0 \}$$ 
be the kernel of the operator $R(\mathcal{L}_{(u,x^i)}(\cdot))$. Note that $R(\mathcal{L}_{(u,x^i)}(\cdot))=\delta R(\varphi(u,x^i))$ since \( R \) is elementwise, linear, 1-Lipschitz, and independent of \( u \).

Since $\mathcal{L}_{(u,x^i)}(\cdot)$ maps an infinite-dimensional space $\cV$ to a finite-dimensional one and $R$ has full row rank, the kernel of $R(\mathcal{L}_{(u,x^i)}(\cdot))$ is infinite-dimensional. We define the intersection of ${\ker}(u,x^i)$ for all $i \leq j$ as follows:
$$
{\ker}(u,\mathcal{X}^j):=\operatorname{span}\{ \delta u \in \cV \mid \delta u \in\bigcap_{ x^i\in \mathcal{X}^j}{\ker}{(u,x^i)} \}
$$

We define the projection of $D_{\delta u} \mathcal{J}_{j+1}(u)$ on a given subspace of functions ${\ker}(u,\mathcal{X}^j)$, denoted by $\operatorname{proj}_{{\ker}(u,\mathcal{X}^j)}D_{\delta u} \mathcal{J}_{j+1}(u)$, as the solution of the following optimization problem:
\begin{equation}\label{eqn:proj_op}
   {\arg\min}_{d(t) \in {\ker}(u,\mathcal{X}^j)} \int_0^1  | d(\tau) -  D_{\delta u} \mathcal{J}_{j+1}(u)) |^2 d\tau 
\end{equation}

From Theorem~\ref{thm:control_w_smp}, we have that $\cap_{i=1}^j U(x^i,y^i)$ is finite-codimension Banach submanifold of $\cV$. Therefore, the set ${{\ker}(u,\mathcal{X}^j)}$ is closed and continuous. Then, the projection $\operatorname{proj}_{{\ker}(u,\mathcal{X}^j)}D_{\delta u} \mathcal{J}_{j+1}(u)$ is well defined. This guarantees that the problem~\eqref{eqn:proj_op} has always a solution.


Now, we can connect the set of control functions $U(x^i,y^i)$ and the map $R(\mathcal{L}_{(u,x^i)}(\cdot))$ via its kernel: 

\begin{corollary}\label{cor:tangent_kernel}Suppose (A\ref{ass:mem}),(A\ref{ass:lpc}),(A\ref{ass:controllable}) hold. Then, the tangent space of $U(x^i,y^i)$ at a control $u^k$ is $\mathrm{ker}(u^k,x^i)$ and, 
\begin{align}
    T_{u^k}( \cap_{i=1}^j U(x^i,y^i)) = \cap_{i=1}^j \mathrm{ker}(u^k,x^i) =: \mathrm{ker}(u^k,\mathcal{X}^j).
\end{align}
\end{corollary}

See Section~\ref{sec:proof} for a proof of Corollary~\ref{cor:tangent_kernel}. This result shows that any update $\delta u^k$ selected within the intersection of ${\ker}(u,x^i)$ for all $i \leq j$ inherently preserves the end-point mapping at previously learned points, since such $\delta u^k$ lies in the tangent space of $\bigcap_{i=1}^j U(x^i,y^i)$. Therefore, we propose the selection of $\delta u^k$ as the projection of $D_{\delta u}\mathcal{J}_{j+1}(u^k)$ onto $ker(u^k,\mathcal{X}^j)$, and more precisely,
\begin{align}\label{eqn:proj}
    \delta u^k := \operatorname{proj}_{{\ker}(u^k,\mathcal{X}^j)}D_{\delta u} \mathcal{J}_{j+1}(u^k)
\end{align}

Now, we can summarize these observations in the following corollary:


\begin{corollary}\label{cor:algorithm}
Suppose (A\ref{ass:mem}),(A\ref{ass:lpc}),(A\ref{ass:controllable}) hold. Let $u^k$ be a control function such that it has memorized the sub-ensemble $(\mathcal{X}^j,\mathcal{Y}^j)$. 
If $\delta u^k$ is selected as 
$$\operatorname{proj}_{\mathrm{ker}(u,\mathcal{X}^j)}D_{\delta u} \mathcal{J}_{j+1}(u),$$
then the control function $u^{k+1}(:=u^k+\delta u^k)$ has been {\em tuned} for $\mathcal{X}^{j+1}$ {\em without forgetting} $\mathcal{X}^j$. 
\end{corollary} 

Paraphrasing the statement says that, under the given assumptions, a control $u^j$ that has memorized the sub-ensemble $(\mathcal{X}^j,\mathcal{Y}^j)$ can be continuously perturbed to a control $u^{j+1}$ that has memorized the sub-ensemble $(\mathcal{X}^{j+1},\mathcal{Y}^{j+1})$ without leaving the set $\bigcap_{i=1}^j U(x^i,y^i)$ (i.e. without forgetting the sub-ensemble ($\mathcal{X}^j,\mathcal{Y}^j)$).  Starting from an empty ensemble \( (\mathcal{X}^0,\mathcal{Y}^0) \) and iteratively adding points one by one, this procedure enables an iterative training of neural ODEs that successively memorizes the full dataset.

    In our previous work~\cite{bayram2024control}, we provided a numerical algorithm implementing this approach along with experimental validation; for details on the algorithm and its empirical performance, we refer the reader to Section 4 of~\cite{bayram2024control}. Also, in~\cite{bayram2025control}, we developed a numerical algorithm that employs TwF for control disturbance rejection, ensuring that the end-point mapping is preserved even under bounded disturbances on the control function. This letter complements these contributions by providing a theoretical foundation for the TwF principle.

\section{Proof of the Main Theorem}\label{sec:proof}

\paragraph{Overview of the Proof:} First, we consider the space of controls that memorizes a given ensemble. Under Assumption~\ref{ass:mem}, we show that $\cap_{i=1}^q U(x^i,y^i)$ is non-empty. Then, we discuss its geometric properties. We define a map that sends a control $u$ to the solution of the system~\eqref{eqn:control_system} at a time $1$ from initial point $x^i$. We show that, under LCP (see Assumption~\ref{ass:lpc}), this map is a submersion from $\cV$ to $\mathbb{R}^{\bar{n}}$. Then, the set of controls that memorizes a pair $(x^i,y^i)$ (i.e. $U(x^i,y^i)$), is a Banach submanifold of $\cV$ from Regular Value Theorem in infinite-dimensional spaces. Then, under the strong memorization property (see Assumption~\ref{ass:controllable}), we show that the intersection of the set of controls that memorizes all the pairs in the ensemble, $\cap_{i=1}^q U(x^i,y^i)$, is also Banach submanifold of $\cV$. Then, we prove that the intersection of the kernel of the map $R(\mathcal{L}_{(u,x^i)}(\cdot))$ for all $x^i \in \mathcal{X}^j$ at a given control $u$ is the tangent space of the intersection of the set of the controls, $\cap_{i=1}^j U(x^i,y^i)$, at a given control $u$. Then, the projection guarantees that the gradient flow is restricted to the submanifold of controls $\cap_{i=1}^j U(x^i,y^i)$ for $1\leq j \leq q$.

\subsection{Geometry of $U(x^i,y^i)$} 

In this subsection, we discuss the geometry of the set $U(x^i,y^i)$ in the view of the following proposition. 

\begin{proposition}\label{prop:one_w_smp} Let $u \in U(x^i,y^i)$ and assume that model~\eqref{eqn:control_system} has the LCP at $x^i$. Then $U(x^i,y^i)$ is a Banach submanifold of $\cV$ of finite-codimension.
\end{proposition}
To prove Proposition~\ref{prop:one_w_smp}, we use the following definitions and theorems. We define the map $G_{x^i}$ for a given initial point $x^i \in \mathcal{X}$ as follows:
\begin{align}\label{eqn:defn_G}
    {G}_{x^i}: \cV \to \mathbb{R}^{\bar{n}} : u \mapsto G_{x^i}(u):=\varphi(u,x^i).  \\[-1.8em]   
\end{align}
In words, the map $G$ sends a control $u$ to the solution at time $1$ of~\eqref{eqn:control_system} with control $u$ and initial point $x^i$. The first notion we need to build upon for Regular value theorem in Banach spaces is submersion. 

{Here, we present an adapted definition of submersion to our context from ~\cite[Chapter 2]{lang2012fundamentals} and~\cite{glockner2015fundamentals}. Let $M$ and $N$ be two $C^r$ smooth manifolds. 

\begin{definition}[Submersion]\label{defn:submersion}
We say that a map $G: M \to N$ is a submersion if for each $m \in M$, there exists open sets $\mathcal{O}_m \subset M$ and $\mathcal{O}_{G(m)} \subset N$, containing $m$ and $G(m)$, and changes of variables $\psi_1:\mathcal{O}_m  \to M $ and $\psi_2:\mathcal{O}_{G(m)} \to N $ with the property that  
\begin{align}\label{eqn:submersion}
    \psi_1 \circ G \circ \psi_2^{-1}
\end{align}
admits a continuous linear right inverse.
\end{definition}

The condition for submersion in Definition~\ref{defn:submersion} is generally difficult to verify. Hence, we provide a simplified version that is relatively easier to check and demonstrate that, under certain conditions, the two are equivalent.

\begin{theorem}\label{thm:naive_in_finite}
Let $G: M \to N$ be a smooth map between $C^r$-manifolds modelled on locally convex topological $\mathbb{R}$-vector spaces. If $N$ is a Banach manifold and $r \geq 2$ or $N$ is finite-dimensional, then $G$ is a $C^r$ submersion if, and only if, for each $m \in M$, the continuous linear map $T_{m}G:T_m M\to T_{G(m)} N$ is surjective. 
\end{theorem}
See \cite[Thm A]{glockner2015fundamentals} for a proof of Theorem~\ref{thm:naive_in_finite}. Now, we can state the regular value theorem in Banach Spaces.

\begin{theorem}[Regular Value Theorem in Banach Spaces]\label{thm:rvt}
    Let $G: M \to N $ be a $C^r$ $\mathbb{R}$-submersion between $C^r$ $\mathbb{R}$-manifolds modelled on locally convex topological $\mathbb{R}$-vector spaces. Let $S$ be a submanifold of $N$. If $S$ has finite codimension $k$ in $N$, then $G^{-1}(S)$ has codimension $k$ in $M$.
\end{theorem}
See~\cite[Thm C]{glockner2015fundamentals} for a proof of Theorem~\ref{thm:rvt}. Figure~\ref{fig:proof_visualization} provides an intuitive illustration of the proof of Proposition~\ref{prop:one_w_smp}, showing the relationships between the sets and spaces involved. Now, we are in a position to prove Proposition~\ref{prop:one_w_smp}.

 \begin{figure}
     \centering
     \includegraphics[width=0.8\linewidth]{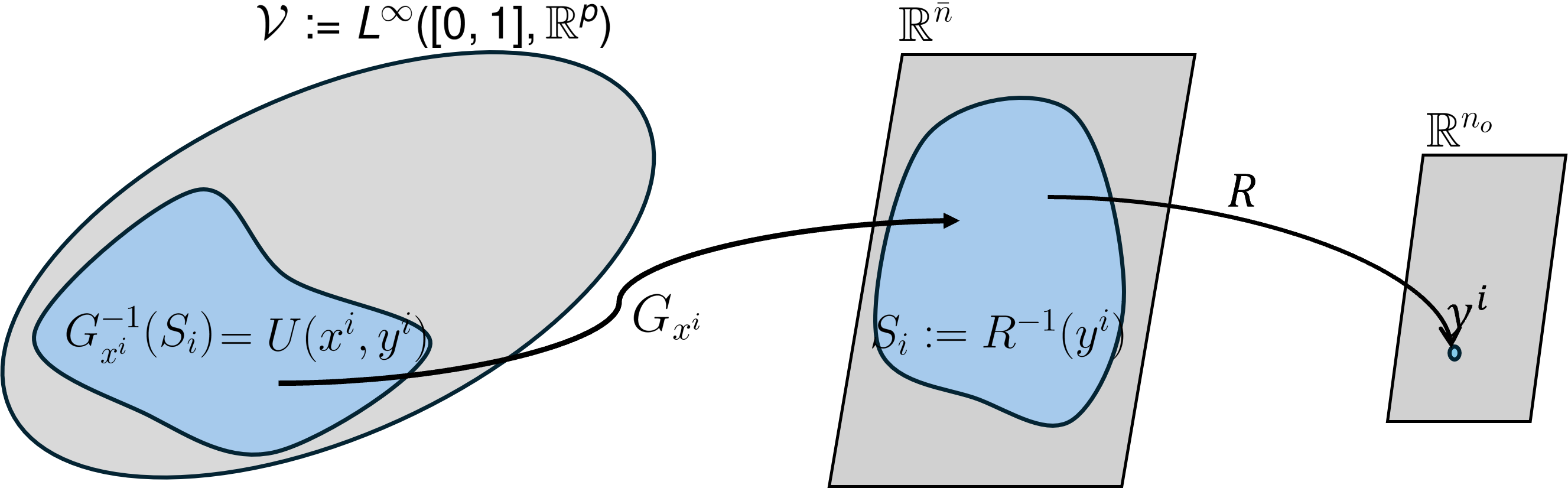}
\caption{
Visualization of the proof approach. This figure illustrates the key spaces and mappings used in the proof. The space of essentially bounded functions, $\mathcal{V}$, is mapped into $\mathbb{R}^{\bar{n}}$ via the submersion $G_{x^i}(\cdot)$, which represents the solution of the model at time 1 starting from $x^i$. The set of control functions that satisfy the end-point condition is given by the preimage $G_{x^i}^{-1}(S_i)$, where $S_i$ is the preimage of the point $y_i$ under the readout map $R$.
}
     \label{fig:proof_visualization}
 \end{figure}

\begin{proof}[Proof of Proposition~\ref{prop:one_w_smp}] First, we show that $G_{x^i}$ is submersion from $\cV$ to $\mathbb{R}^{\bar{n}}$. Then, we apply Regular value theorem in Banach spaces to show that $G_{x^i}^{-1}(\cdot)$ is a Banach submanifold $\cV$ with finite-codimension.

 Let $z_1 \in \mathbb{R}^{\bar{n}}$. We need to show that there exists a linear continuous map $\eta(z_1)$ so that $dG_{x^i}(\eta(z_1))=z_1$ to show that $G_{x^i}$ is a submersion (see Definition~\ref{defn:submersion}). One can see that the differential $dG_{x^i}\lvert_u(v)$, which entails the first-order variation of $G_{x^i}$ at $u$, that is, $$dG_{x^i}\lvert_u(v):=G_{x^i}(u+v)-G_{x^i}(u)$$ is the solution of~\eqref{eqn:defn_ltv} at time $1$. Then, we have:
\begin{equation}\label{eqn:dG_defn}
dG_{x^i}\lvert_u(v) = \int_0^1 \Phi_{(u,x^i)}(1,s) B(s) v(s) ds.    
\end{equation}
where $B(s)=\frac{\partial f(x,u)}{\partial u}|_{x=\varphi_s(u,x^i)}$.
Then, from the linearized controllability property of the system~\eqref{eqn:control_system} at $x^i$, we know that the LTV system~\eqref{eqn:defn_ltv} is controllable. This implies that the controllability Gramian of~\eqref{eqn:defn_ltv}, denoted by $W(0,1)$, is full-rank~\cite{brockett2015finite}. Let $\chi(z_1)=-W(0,1)^{-1}\Phi_{(u,x^i)}(0,1)z_1$, and
$$\eta_{z_1}(s)=-B^\top(s)\Phi_{(u,x^i)}^\top(0,s)\chi(z_1).$$ One can also see that $\eta_{z_1}$ is continuous and linear in $z_1$. When, we plug these two into~\eqref{eqn:dG_defn}, we obtain:
    \begin{align}
        dG_{x^i}\lvert_u(\eta_{z_1}) &= \int_0^1 \Phi_{(u,x^i)}(1,s)B(s)\eta_{z_1}(s) ds \\
                    &= -\Phi_{(u,x^i)}(1,0) W(0,1) \chi(z_1) = z_1
    \end{align}
This shows that $\eta_{z_1}(s)$ is the linear continuous right-inverse of $dG_{x^i}$ at $z_1$, implying that $dG_{x^i}$ is surjective (i.e. every vector in $\mathbb{R}^{\bar{n}}$ can be reached by some element in the range of $dG_{x^i}$). Note that the codomain of $G_{x^i}$ is $\mathbb{R}^{\bar{n}}$, which is a finite-dimensional Banach manifold. The map $G_{x^i}$ is defined as a flow map of an ODE; it inherits the smoothness of the vector field generating the flow. Then, from Theorem~\ref{thm:naive_in_finite}, the map $G_{x^i}$ is submersion.

Recall the definition of $U(x^i,y^i)$: 
\begin{align}
U(x^i,y^i):&=\{ u \in \cV \mid \varphi(u,x^i) \in R^{-1}(y^i) \} 
\end{align}
By substituting $G_{x^i}(u) := \varphi(u,x^i)$ and $S_i := R^{-1}(y^i)$, we obtain
\begin{align}
    U(x^i,y^i)  &=\{ u \in \cV \mid G_{x^i}(u) \in S_i\}
\end{align}
Hence, we have $U(x^i,y^i)=G_{x^i}^{-1}(S_i)$. From Theorem~\ref{thm:rvt}, we can conclude that $G_{x^i}^{-1}(S_i)$ is a Banach submanifold of $\cV$ of codimension equal to the codimension of $S_i$ in $\mathbb{R}^{\bar{n}}$, (i.e. $\bar{n}-\dim S_i$). Note that, by definition, the Jacobian of $R$ is of full row rank. From~\cite{guillemin2010differential}, we have that $\dim S_i=\bar{n}-n_o$ as a result of regular value theorem. Then, we have $\mathrm{codim}(U(x^i,y^i))=n_o$. This concludes the proof. \qed
\end{proof}

We can now complete the proof of Corollary~\ref{cor:tangent_kernel}. 
\begin{proof}[Proof of Corollary~\ref{cor:tangent_kernel}]
    Recall that $U(x^i,y^i)=G^{-1}_{x^i}(S_i)$. From~\eqref{eqn:dG_defn}, it follows that $T_{u^k}(U(x^i,y^i))= \mathrm{ker} R(dG_{x^i}\lvert_{u^k})$. Since the tangent space of an intersection of submanifolds is the intersection of their tangent spaces, the result follows immediately. \qed
\end{proof}

\subsection{Differentiable Properties of $\cap_{i=1}^q U(x^iy^i)$}

Until now, we have only considered the memorization property (A\ref{ass:mem}) and the LCP assumption (A\ref{ass:lpc}) for the proof of Theorem~\ref{thm:control_w_smp}. Now, we take the advantage of strong memorization property to discuss the intersection of the set of controls $U(x^i,y^i)$ for all $i \in \mathcal{I}$. It is important to note that the transversal intersection of infinite-dimensional Banach spaces does not necessarily yield a Banach submanifold, whereas the transversal intersection of finite-dimensional Banach manifolds does~\cite{guillemin2010differential}. Therefore, we need to further show that under the strong memorization property, the intersection of \( U(x^i, y^i) \) for $i=1,2,\cdots,q$ is also a Banach submanifold of $\cV$ with finite codimension.

A closed subspace $M_1$ of the Banach space $M$ is said to be split (or complemented) in $M$, if there is a closed subspace $\bar{M}_1$ such that $M = M_1 \oplus \bar{M}_1$ and $M_1 \cap \bar{M}_1 = \{0\}$. With this definition, we can present a corollary of Transversal Mapping Theorem, adapted to our context.

\begin{theorem}[Transversal Mapping Theorem]\label{thm:splits}
    Let $M_1$ and $M_2$ be submanifolds of $M$. Suppose that:
    \begin{enumerate}[label=(\roman*)]
        \item\label{itm:transversal} $T_m M_1 + T_m M_2 = T_m M$ for all $m \in M_1 \cap M_2$.
        \item\label{itm:split} $T_m M_1 \cap T_m M_2$ splits in $T_m M$ for all $m \in M_1 \cap M_2$. 
    \end{enumerate}
    If both $M_1$ and $M_2$ have finite codimension in $M$, then 
    \[
    \mathrm{codim}(M_1 \cap M_2) = \mathrm{codim}(M_1) + \mathrm{codim}(M_2).
    \]
\end{theorem}

See~\cite[Corollary 3.5.13 and Theorem 3.5.12]{abraham2012manifolds} for a proof of Theorem~\ref{thm:splits}.  

In other words, if two submanifolds intersect transversally (see~\ref{itm:transversal}) and the intersection of their tangent spaces splits the ambient tangent space (see~\ref{itm:split}), then their intersection is itself a submanifold, and its codimension is the sum of their individual codimensions.

\begin{proof}[Proof of Theorem~\ref{thm:control_w_smp}]
Under assumption (A.\ref{ass:lpc}), by Proposition~\ref{prop:one_w_smp}, we know that $U(x^i,y^i)$ is a Banach submanifold of $\cV$ of finite codimension for $i = 1, \cdots, q$. Under assumptions (A.\ref{ass:mem}) and (A.\ref{ass:controllable}), there exists a control $u \in U(x^1,y^1) \cap U(x^2,y^2)$ for the model~\eqref{eqn:control_system}, and the sets of control functions $U(x^1,y^1)$ and $U(x^2,y^2)$ intersect transversally. Then, $T_{u} U(x^1,y^1)$ and $T_{u} U(x^2,y^2)$ satisfy~\ref{itm:transversal} in Theorem~\ref{thm:splits} for all $u \in U(x^1,y^1) \cap U(x^2,y^2)$. 

From Corollary~\ref{cor:tangent_kernel}, we interchangeably use $\mathrm{ker}(dG_{x^i}\lvert_u)$ and $T_uU(x^i,y^i)$. By Proposition~\ref{prop:one_w_smp}, we know that $\mathrm{ker}(dG_{x^i}\lvert_u)$ is closed and continuous for all $i \in \mathcal{I}$ and has finite codimension in $\cV$. Then, the intersection of $\mathrm{ker}(dG_{x^1}\lvert_u)$ and $\mathrm{ker}(dG_{x^2}\lvert_u)$ is also closed and continuous from~\cite{rudin1964principles}. Let $A := \mathrm{ker}(dG_{x^1}\lvert_u) \cap \mathrm{ker}(dG_{x^2}\lvert_u)$. 
Define the quotient map $\pi: T_u\cV \to T_u\cV/A$. Let $\{e_1, e_2, \ldots, e_\ell\}$ be a basis for $T_u\cV/A$ (note that we know that $\ell$ is finite and less than $2\bar{n}$). Pick $m_i \in T_u\cV$ so that $\pi(m_i) = e_i$ for $1 \leq i \leq \ell$, and let $N$ be the vector space spanned by $\{m_1, \cdots, m_\ell\}$~\cite[Lemma 4.21]{rudin1964principles}. Then,
\[
T_u\cV = N \oplus A.
\]
In words, $A = T_u U(x^1,y^1) \cap T_u U(x^2,y^2)$ splits $T_u\cV$ (see~\ref{itm:split}). Then, from Theorem~\ref{thm:splits}, we conclude that $U(x^1,y^1) \cap U(x^2,y^2)$ is a Banach submanifold of finite codimension. Iterating a finite number of times proves the result.\qed
\end{proof}

\section{Summary and Future Work}\label{sec:conl}

In this letter, we have taken a geometric approach to Tuning without Forgetting (TwF), an iterative training algorithm for neural ODEs, in which new training points are added sequentially, and parameters are updated within the subspace of control functions that preserve the end-point mapping at previously learned samples. This framework naturally extends to applications in continual learning, robust learning, and federated learning. We have shown that exact mapping preservation (i.e., not forgetting) during sequential training is guaranteed because the parameter subspace forms a Banach submanifold of finite codimension under nonsingular controls. This revealed the connection between TwF and nonsingular loops, a concept extensively studied in classical control theory.

For future work, we plan to extend our study to normalized flows induced by a neural ODE. In the current work, we had considered a finite set of training points; in the future, we aim to generalize $\mathcal{X}$ to a continuum or probability distribution, incorporating mapping-preserving operations over distributions by utilizing the null set of the input measure.

\vspace{-3mm}
\bibliographystyle{ieeetr}
\bibliography{learning.bib}

\end{document}